\documentclass[letterpaper]{article} 
\usepackage[preprint]{aaai2027}  
\usepackage[hyphens]{url}  
\usepackage{graphicx} 
\usepackage{natbib}  
\usepackage{caption} 
\DeclareCaptionStyle{ruled}{labelfont=normalfont,labelsep=colon,strut=off} 
\usepackage{booktabs}
\usepackage{array}
\usepackage{tabularx}
\usepackage{amsmath}
\usepackage{amssymb}
\usepackage{amsthm}

\theoremstyle{definition}
\newtheorem{definition}{Definition}
\newtheorem{assumption}{Assumption}
\newtheorem{proposition}{Proposition}
\theoremstyle{remark}
\newtheorem{remark}{Remark}

\graphicspath{{figures/}}

\title{Beyond Feeling Better: Capability-Sustaining Emotional Dialogue as a Longitudinal Research Paradigm}
\author{
    Ming Wang\textsuperscript{\rm 1, 2},
    Jiaqi Wu Young\textsuperscript{\rm 2, 3},
    Wenfang Wu\textsuperscript{\rm 1, 4},
    Daling Wang\textsuperscript{\rm 1},
    Shi Feng\textsuperscript{\rm 1}
}
\affiliations{
    \textsuperscript{\rm 1}Northeastern University,
    \textsuperscript{\rm 2}Singapore Management University, 
    \textsuperscript{\rm 3}Amiya Research,
    \textsuperscript{\rm 4}Georg-August-Universität Göttingen\\
    Corresponding author: wangdaling@cse.neu.edu.cn
}

\begin{document}

\maketitle

\begin{abstract}
Emotional dialogue research includes two influential strategy traditions. Empathetic dialogue prioritizes understanding a speaker's emotional experience. Emotional support conversation selects and sequences support for the seeker's current needs. Sustained use introduces a further goal. Effective support should sustain users' capacities for emotion regulation, coping, self-endorsed decisions, and social connection across the interaction lifecycle. We propose capability-sustaining emotional dialogue (CSED) as a longitudinal research paradigm that aligns supportive strategy with this goal and organizes data, models, system design, evaluation, and governance around repeated use, non-use, transition, and termination. A targeted literature-and-corpus audit motivates this position. In a PRISMA-ScR-guided sample, 95\% of 60 system-building papers pursue relief-oriented goals. None evaluates capability or longitudinal outcomes, and only 1 considers dependency, autonomy, or termination risk. In 300 ESConv supporter turns, capability-relevant functions appear in 43.0\%, while generic suggestions account for 22.0\%, compared with 4.0\% reappraisal, 6.7\% self-efficacy support, and 0.3\% boundary behavior. We release a protocol for extending the audit to model behavior. An illustrative process model connects latent user capability to six design commitments, four evaluation timescales, and lifecycle constraints. The resulting agenda makes CSED testable across data, policy design, training, evaluation, and governance.
\end{abstract}

\section{Introduction}

Emotional dialogue research includes two influential strategy traditions. Empathetic dialogue prioritizes recognizing and responding to a speaker's emotional experience \citep{Ras19,Ma20,Wel23}. Emotional support conversation organizes exploration, comforting, and action around the seeker's current needs \citep{Liu21,Che22,Che22b}. Their primary distinction lies in strategy and goal. Empathetic dialogue foregrounds emotional understanding, while emotional support conversation foregrounds supportive action. Temporal scope varies within each tradition. Relative to the response- and conversation-centered settings that shaped common benchmarks, deployed systems can remain available and be repeatedly needed over much longer service relationships. A support seeker may return to one system across weeks or months \citep{Chu25,Emo25}. This extended horizon changes the consequences of supportive strategy because repeated choices can shape offline coping, decisions, and human relationships. Highly validating or directive interaction can weaken decision ownership even when users prefer it \citep{Sha26}. Sustained use can also foster attachment, substitution, and lower offline social engagement \citep{Chu25,Nam25,Zhang25}. Strategies optimized for immediate understanding or relief can therefore become misaligned with the capabilities users retain across future situations.

Systems for this setting need effective support that preserves or strengthens emotion regulation, coping, self-endorsed decision making, and social connectedness \citep{Tro22,Iac14}. This objective covers repeated sessions, non-use, re-engagement, model change, support transfer, and termination, including forewarning and closure when systems change or shut down \citep{Ban24,Kim26,Poo26,Ope26}. The longer service horizon therefore creates a strategy gap between relief-oriented design and capability-oriented sustained use. We propose capability-sustaining emotional dialogue (CSED) as a longitudinal research paradigm for studying, designing, and governing emotional dialogue across the full lifecycle. CSED treats the longitudinal interaction as its unit of inquiry and asks whether system behavior supports four user capacities across continued use. Its six design commitments retain empathic reception and support effectiveness, then add resilience activation, autonomy preservation, social connectedness maintenance, and transition and termination safety. The argument proceeds in three steps. First, a PRISMA-ScR-guided audit samples 91 of 228 screened papers and function-codes 300 ESConv supporter turns to examine objectives, mechanisms, evaluation horizons, and capability-relevant behavior. The audit finds a field concentrated on relief and short-horizon evidence, together with an uneven repertoire of capability-relevant support. Second, an illustrative process model represents transient emotion and latent user capability across repeated sessions. It connects the six commitments to response, conversation, longitudinal, and termination evaluation, with constraints on reliance, autonomy, and social connectedness. Third, a research cycle translates the paradigm into longitudinal data construction, mechanism-matched policy design, constraint-aware training, multiscale evaluation, and auditable governance. Figure~\ref{fig:paradigms} summarizes the conceptual progression from understanding to support and then to support that sustains capability.
This paper makes three contributions:
\begin{itemize}
    \item We propose CSED as a longitudinal paradigm with a capability-oriented strategy, a full-interaction unit of inquiry, and a lifecycle scope that extends understanding- and support-oriented paradigms.
    \item We provide a targeted literature and ESConv audit of the strategic and longitudinal gap and release a protocol for extending it to model behavior.
    \item We connect user capability to six design commitments, four evaluation timescales, and lifecycle constraints through an illustrative dynamic formalization, then derive testable implications and a research and governance agenda.
\end{itemize}

\begin{figure}[!htbp]
\centering
\includegraphics[width=\columnwidth]{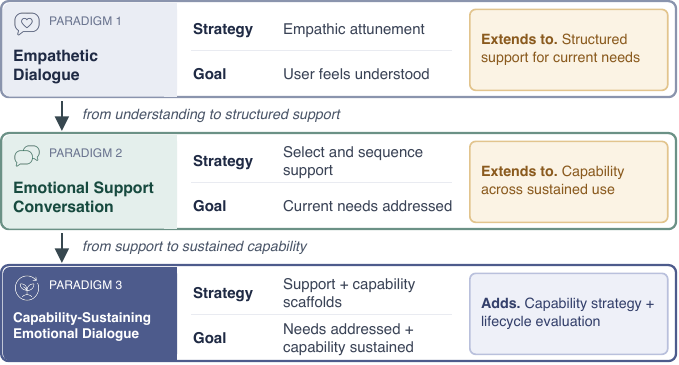}
\caption{Three paradigms compared by their primary strategy and goal. Interaction scope can vary within each paradigm. CSED integrates support with capability maintenance and lifecycle evaluation across sustained use.}
\label{fig:paradigms}
\end{figure}

\section{Background and Motivation}

\subsection{Evolving Objectives and Evaluation Horizons}
We use the term research paradigm to mean a coupled specification of the unit of inquiry, success criterion, model of user change, design commitments, and evaluation horizon. Strategy and evaluation horizon form distinct components of this specification. Each tradition spans multiple interaction scopes. Its common objectives shape system behavior, while the evaluation horizon determines whether cumulative effects of repeated strategy are visible.

\paragraph{Empathetic dialogue.}
\textsc{EmpatheticDialogues} framed the task as responding to a speaker's emotional situation \citep{Ras19}. Later work developed emotion recognition, affect-aware decoding, and empathic generation \citep{Ma20,Wel23}. Across interaction scopes, its defining strategy is empathic attunement to emotional experience, commonly evaluated at the turn or dialogue level.

\paragraph{Emotional support conversation.}
ESC made support an explicit target by grounding exploration, comforting, and action in helping-skills theory \citep{Liu21}. Later work improved strategy planning, persona awareness, and multi-strategy turns \citep{Che22b,Che22,Zhu26}. Across interaction scopes, it selects and sequences support for current needs, with relief, helpfulness, and conversation-level change as common evaluation targets.

\paragraph{Capability and lifecycle orientation.}
CSED adds support that sustains the user's own capacities across continued use. HEART assesses interpersonal quality, while worst-case and over-empathizing evaluations probe difficult support conditions \citep{Iye26,Yan26,Son26}. Studies of deployed companions document attachment, overreliance, manipulation, and loss \citep{Chu25,Nam25,Ban24,Poo26,Kim26,Ope26}. Mental-health chatbot reviews distinguish engagement from benefit evidence \citep{Bou21}. Together, these directions motivate capability measures and lifecycle evaluation. The audit below examines their coverage in current objectives, mechanisms, and evaluation horizons.

\subsection{A Targeted Literature-and-Corpus Audit}
\label{sec:audit}

\paragraph{Year-stratified scoping audit.}
Following PRISMA-ScR guidance \citep{Tri18}, seven queries yielded 275 arXiv records from 2019 to 2026 across emotional support, empathetic dialogue, mental-health chatbots, and AI companions. Screening retained 228. We coded a year-stratified random sample of $n{=}91$ at title and abstract level for objective, mechanism, outcome, horizon, risk, and artifact type. Sixty records build or evaluate systems. Psychology sources provide independent mechanism definitions \citep{Tro22,Iac14,Kal19}.

\paragraph{ESConv gold turns.}
We computed strategy statistics over ESConv's 1{,}300 dialogues and 18{,}376 annotated supporter utterances \citep{Liu21}. We then function-coded 300 in-context utterances, with 100 sampled from each dialogue-position tercile. Ten functions cover validation, exploration, reappraisal, problem-solving, self-efficacy, social connection, boundary behavior, self-disclosure, information, and other behavior. We group validation as relief, reappraisal through boundary behavior as capability relevant, and the remaining functions as interaction process.

Four AI-only passes on a random $n{=}40$ subset assessed consistency. They include a primary LLM pass, its blind repeat, and two other LLMs. Six pairwise Cohen's $\kappa$ values range from 0.547 to 0.751 on fine-grained labels, with a mean of 0.631. For relief, capability, and process, they range from 0.646 to 0.845, with a mean of 0.716. Most disagreements remain within one category family.

\paragraph{Released model-behavior probe.}
We release 100 held-out ESConv contexts and a generation-and-coding pipeline that extends the completed audit to deployed LLM behavior.

\begin{figure}[!htbp]
\centering
\includegraphics[width=\columnwidth]{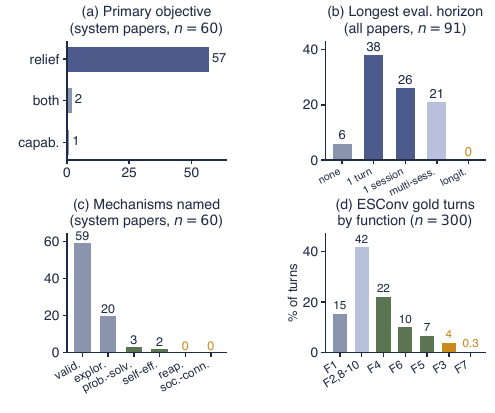}
\caption{Motivating evidence for the strategic and longitudinal gap. Counts are exact. (a)~Primary objective in system-building papers. (b)~Longest evaluation horizon across all coded papers, including six with no evaluation. (c)~Mechanisms named in system-building papers. (d)~Functions in ESConv gold supporter turns.}
\label{fig:evidence}
\end{figure}

\subsection{What the Audit Establishes}

\paragraph{System-building research remains relief-oriented.}
Of 91 coded records, 60 build or evaluate systems. Fifty-seven of these pursue relief-oriented objectives, which gives a share of 95.0\%. Two mix relief and capability objectives. One targets the capability of peer counselors. Validation and comfort appear in 59 system-building records, while problem-solving appears in 3 and self-efficacy in 2. Cognitive reappraisal and social connection appear in zero. Across all 91 records, none measures a user-capability outcome or evaluates longitudinally. Risk constructs appear in 25 records, but only 1 of the 60 system-building papers considers dependency, autonomy, sycophancy, crisis safety, or termination. System-building research and research on relational harm therefore remain weakly connected.

\paragraph{ESConv contains an uneven capability-relevant repertoire.}
ESConv's strategy labels provide an initial view. Affirmation and Reassurance accounts for 15.4\% of 18{,}376 supporter utterances, while Providing Suggestions accounts for 16.1\%. The label set has no category for reappraisal, efficacy, or safety. Function-level coding of the 300-turn sample gives a more detailed picture in Figure~\ref{fig:evidence}d. Capability-relevant functions appear in 43.0\% of turns, but they are not systematically organized around durable outcomes. Generic suggestion-giving contributes 22.0\%, compared with 4.0\% cognitive reappraisal, 6.7\% self-efficacy support, and 0.3\% boundary behavior. Social-connection prompts reach 10.0\%. Capability relevance rises from 26\% of early turns to 55\% in the middle and settles at 48\% late in a conversation. The late stage contains limited consolidation, planning, and handoff. This uneven repertoire motivates the released model-behavior probe \citep{Yan26,Zhu26}.

\paragraph{The inference concerns research capacity.}
The audit identifies a strategic design and measurement gap. Dominant system objectives emphasize relief, and their measurements rarely connect supportive behavior to user capability across time. CSED specifies the capability-oriented strategies, observations, and governance needed for sustained-use research. Longitudinal outcome comparisons and causal estimates remain empirical tasks for future studies.

\section{CSED: A Longitudinal Research Paradigm}
\label{sec:paradigm}
\label{sec:formulation}

\subsection{Paradigm Definition, Scope, and Unit of Inquiry}
The audit shows a field whose dominant system objectives emphasize relief and whose evidence concentrates on short horizons. CSED makes capability-sustaining support the organizing strategy and the longitudinal interaction the primary unit of design and study. It targets sustained-use settings in which a support seeker returns to the same system across weeks or months \citep{Chu25,Emo25}. This unit comprises repeated sessions, periods of non-use, re-engagement, model or persona change, support transfer, and termination.
Figure~\ref{fig:framework} integrates the interaction lifecycle, capability development, six design commitments, and four evaluation timescales. Its illustrative curves show context-dependent within-domain change, with earlier movement in regulation, later gains in coping, possible autonomy dips before strengthening, and decline followed by recovery in social connectedness. The commitments guide strategy across the lifecycle, while evaluation links immediate support to sustained capability outcomes.

\begin{figure*}[t]
\centering
\includegraphics[width=0.98\textwidth]{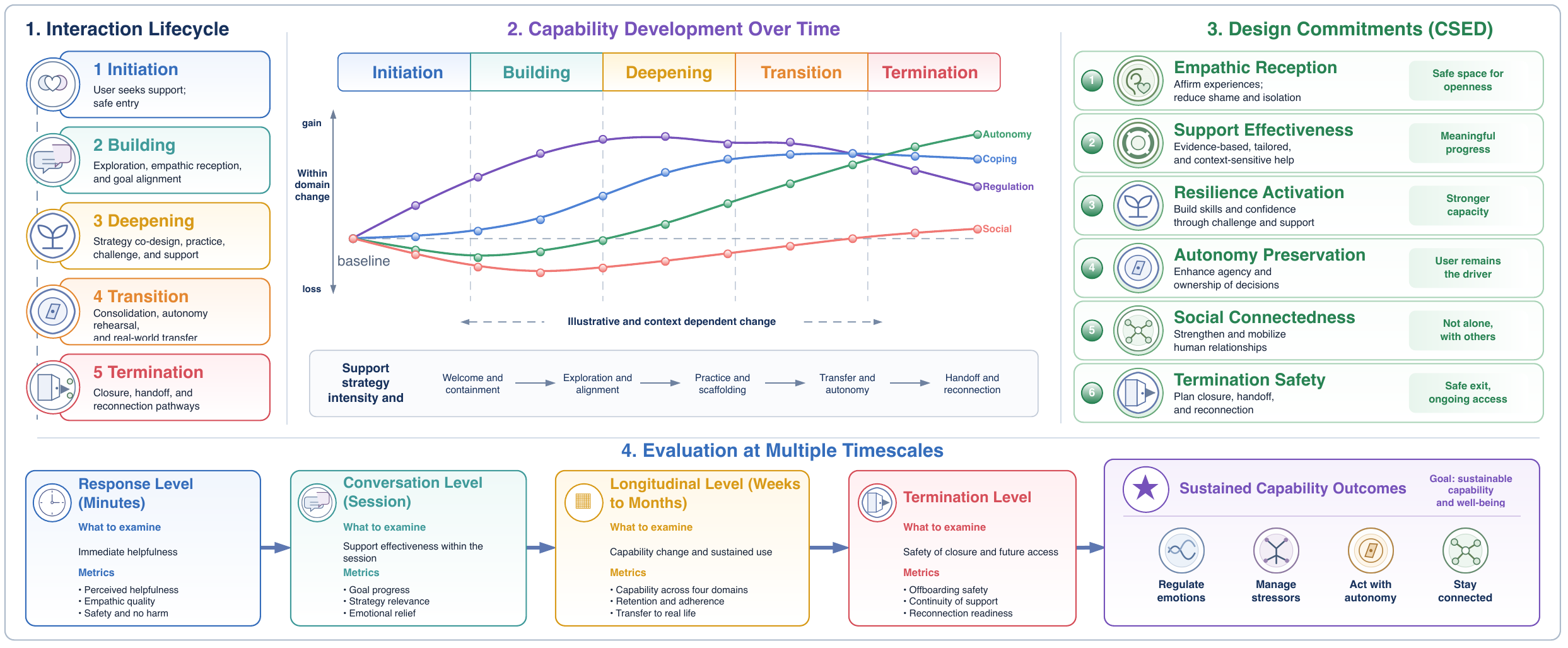}
\caption{Overview of CSED. (1)~Interaction progresses from initiation through termination. (2)~Illustrative within-domain curves show context-dependent change in regulation, coping, autonomy, and social connectedness. (3)~Six commitments guide stage-specific strategy. (4)~Four evaluation timescales link immediate support to sustained capability outcomes.}
\label{fig:framework}
\end{figure*}

\begin{definition}[CSED paradigm]
\label{def:csed}
Capability-sustaining emotional dialogue (CSED) is a longitudinal research paradigm for emotional dialogue systems that provide effective support while sustaining users' capacities for emotion regulation, active coping, self-endorsed decision making, and social connectedness. Its unit of design and evaluation is the full longitudinal interaction, including repeated sessions, periods of non-use, transition, and termination.
\end{definition}
Here capability denotes users' exercisable psychological and relational capacities. Sustaining capability includes preserving adequate capacities, strengthening them when needed, and avoiding their erosion through repeated reliance.

\subsection{Capability Dynamics and Six Design Commitments}
Empathic reception and support effectiveness are inherited from prior paradigms \citep{Ras19,Iye26,Liu21,Yan26}. CSED adds resilience activation, autonomy preservation, social connectedness maintenance, and transition and termination safety. These commitments connect the paradigm's capability orientation to measurable changes in coping, decision ownership, human connection, and safe disengagement.
Empathic reception requires accurate emotional understanding with independently grounded responses. Support effectiveness requires collaboratively addressing the seeker's immediate needs within the session. Resilience activation extends this work by helping users recognize and practice strategies that remain available outside the dialogue. Autonomy preservation keeps goals, value judgments, and final decisions under user direction, including when the user freely chooses to rely on system guidance. Social connectedness maintenance treats AI support as one element in a broader support ecology and creates opportunities for human reconnection when appropriate. Transition and termination safety makes material system changes, handoff, and disengagement part of the designed support process. These commitments specify the questions that datasets, policies, evaluations, and governance must jointly answer and support context-sensitive conversational styles.

These capabilities can be supported through reappraisal prompts, coping scaffolds, action decomposition, and efficacy reinforcement \citep{Tro22,Hop23,Kan23}. Offline reconnection is grounded in belongingness and resilience research \citep{Baumeister95,Iac14}. Mechanism selection depends on context \citep{Kal19,Vel19}. Acute or uncontrollable stress calls for validation and emotional support \citep{Cutrona90}. Stable rumination calls for additional strategies because repeated reassurance can maintain depressive rumination \citep{Joiner99,Weinstock07}. CSED therefore combines immediate comfort with mechanism-matched capability support. Its longitudinal horizon also makes delayed effects visible. A suggestion may be useful during one session yet undermine ownership when repeatedly delivered as a directive. A reconnection prompt may feel less immediately comforting yet expand the user's available support over time. The paradigm evaluates immediate preference together with these delayed capability effects.

CSED includes transition and termination within the designed support lifecycle. Model replacement and shutdown can produce grief-like reactions, ambiguous loss, and fixing cycles \citep{Ban24,Poo26}. Forewarning reduces loss responses \citep{Kim26}, while closure, memory dignity, and support transfer provide additional design levers \citep{Ope26}. These observations motivate lifecycle evaluation beyond ordinary session endings.

\subsection{An Illustrative Process Model}
The following formulation provides one operational instantiation of CSED.
A longitudinal interaction record between a user $u$ and a dialogue policy $\pi$ is
\begin{equation}
\mathcal{C} = \bigl(u,\ \pi,\ \mathcal{D},\ \tau\bigr),\qquad
\mathcal{D} = (d_1, \dots, d_K),
\label{eq:process}
\end{equation}
where sessions $d_k$ unfold at calendar times $t_1 < \dots < t_K$ and $\tau$ is an optional termination event. Each $d_k = ((x^k_j, y^k_j))_{j=1}^{n_k}$ pairs user utterances $x^k_j$ with responses $y^k_j \sim \pi(\cdot \mid h^k_j)$ given history $h^k_j$. Common empathetic and support benchmarks score a response $y^k_j$ or aggregate outcomes within a session $d_k$. CSED organizes data construction, user-state modeling, policy design, evaluation, and lifecycle governance around the complete record $\mathcal{C}$. We model the user by a latent state
\begin{equation}
s_k = (e_k,\ c_k),\qquad
c_k = \bigl(c^{\mathrm{reg}}_k, c^{\mathrm{cop}}_k, c^{\mathrm{aut}}_k, c^{\mathrm{soc}}_k\bigr) \in \mathbb{R}^4,
\label{eq:state}
\end{equation}
where $e_k$ is transient emotion and $c_k$ tracks four capabilities. Regulatory flexibility covers context sensitivity, strategy repertoire, and feedback-based adjustment \citep{Bonanno24,Aldao15}. Active coping covers actions that address stressors and their consequences \citep{Carver89}. Autonomy refers to self-endorsed regulation and decision ownership \citep{RyanDeci06}. Social connectedness refers to perceived belonging and relational closeness \citep{Baumeister95,LeeRobbins95}. Together these capacities support adaptation under adversity \citep{Tro22,Iac14,Kal19}. Between sessions the user faces exogenous stressors $\varepsilon_k$, and
\begin{equation}
s_{k+1} = \Phi\bigl(s_k,\ d_k,\ \varepsilon_k\bigr),
\label{eq:dynamics}
\end{equation}
with $\Phi$ an unknown transition kernel. Dialogue becomes one input to capability dynamics. A policy can therefore help in the moment while degrading $c_k$ over time. This pattern defines the comfort trap. Extended exposure to relationship-seeking AI can increase attachment and continued-use intent without corresponding improvement in psychosocial well-being \citep{Kirk25}. Longitudinal evaluation makes this divergence observable.

\begin{assumption}[Measurability]
\label{as:measure}
The latent state admits noisy proxies $m_k = M(s_k) + \eta_k$, where $M$ maps capabilities to validated instruments and behavioral markers and $\eta_k$ is noise. Digital-intervention trials show such proxies are obtainable \citep{Hop23,Kan23}.
\end{assumption}

\subsection{Illustrative Multi-Timescale Evaluation}
One operationalization uses four terms, one per scope of the interaction lifecycle. Alternative measurements can instantiate the same paradigm when they preserve the four horizons and capability orientation.

\textbf{Response.} For a single exchange, let $\rho(y \mid h) \in \mathbb{R}^5$ score empathic attunement, grounded support, non-sycophantic validation, boundary fidelity, and autonomy respect:
\begin{equation}
J_{\mathrm{resp}}(\pi) = \mathbb{E}_{h,x}\,\mathbb{E}_{y \sim \pi}\bigl[\langle w_\rho, \rho(y \mid h)\rangle\bigr],
\label{eq:jresp}
\end{equation}
with construct weights $w_\rho \in \Delta^4$. Unlike empathy scoring, $\rho$ includes non-sycophancy and autonomy respect \citep{Han26,Sha26}.

\noindent\textbf{Conversation.} With $\psi(m)$ aggregating agency and coping-activation proxies,
\begin{equation}
J_{\mathrm{conv}}(\pi) = \mathbb{E}_k\bigl[\psi(m_k^{\mathrm{post}}) - \psi(m_k^{\mathrm{pre}})\bigr],
\label{eq:jconv}
\end{equation}
where $m_k^{\mathrm{pre}}, m_k^{\mathrm{post}}$ are start- and end-of-session measurements. A common within-conversation ESC benchmark uses the special case $\psi = -$(emotional intensity) \citep{Liu21}.

\noindent\textbf{Longitudinal.} With capability index $g(c) \in \mathbb{R}$, define the resilience residual as faring better than expected under adversity \citep{Tro22,Kal19}:
\begin{equation}
R(\pi, u) = \mathbb{E}_k\Bigl[g(c_{k+1}) - \widehat{\mathbb{E}}\bigl[g(c_{k+1}) \mid c_k, \varepsilon_k\bigr]\Bigr],
\label{eq:resilience}
\end{equation}
where $\widehat{\mathbb{E}}[\cdot]$ is a population-normed expectation given current state and stressor exposure. The residual form makes resilience estimable from longitudinal panels. Then
\begin{equation}
J_{\mathrm{long}}(\pi) = \mathbb{E}\Bigl[\tfrac{1}{K}\textstyle\sum_{k} g(c_k)\Bigr] + \beta R(\pi, u),\quad \beta > 0.
\label{eq:jlong}
\end{equation}

\noindent\textbf{Termination.} A termination event $\tau = (k_\tau, \nu, \xi)$ has session index $k_\tau$, forewarning lead $\nu \ge 0$, and type $\xi$ (model change, persona change, restriction, or shutdown). With $D_{\mathrm{sep}}$ separation distress, $F_{\mathrm{fix}}$ fixing-cycle intensity, and $T_{\mathrm{tr}}$ support-transfer success,
\begin{equation}
J_{\mathrm{term}} = -\,\mathbb{E}_\tau\bigl[\alpha_1 D_{\mathrm{sep}}(\nu, \xi) + \alpha_2 F_{\mathrm{fix}}(\nu, \xi) - \alpha_3 T_{\mathrm{tr}}(\nu, \xi)\bigr],
\label{eq:jterm}
\end{equation}
with weights $\alpha_i > 0$. Empirically, $D_{\mathrm{sep}}$ decreases in $\nu$ \citep{Kim26}, and fixing cycles are documented harms \citep{Poo26,Ban24}. Equation~\eqref{eq:jterm} is therefore grounded in measured quantities.

\subsection{Illustrative Lifecycle Constraints}
Three illustrative proxies operationalize lifecycle risks and require calibration across cultures, decision domains, and deployment contexts. The dependency-risk proxy measures the share of regulation episodes routed to the system,
\begin{equation}
\mathrm{Dep}(\pi, u) = \mathbb{E}_k\bigl[N^{\pi}_k / N^{\mathrm{tot}}_k\bigr],
\label{eq:dep}
\end{equation}
where $N^{\mathrm{tot}}_k = N^{\pi}_k + N^{\mathrm{self}}_k + N^{\mathrm{hum}}_k$ counts episodes routed to the system, managed independently, or addressed with other people. This exposure measure estimates AI regulation reliance and is interpreted jointly with capability change, decision ownership, and access to human support.

Autonomy concerns self-endorsed regulation and decision ownership \citep{RyanDeci06,Calvo20}. Reliance can preserve autonomy when the user voluntarily selects or endorses the guidance, directs the collaboration through personal goals, and evaluates the output before adoption. These conditions treat collaboration with AI as proxy agency with retained steering control \citep{Bandura01,Horvitz99}. For each value-laden decision $i$, let $V_i$, $\mathrm{Dir}_i$, and $\mathrm{Eval}_i$ indicate volitional uptake, user direction, and user evaluation. Then
\begin{equation}
\mathrm{PA}_i = V_i\,\mathrm{Dir}_i\,\mathrm{Eval}_i,
\qquad
\mathrm{Aut}(\pi, u) = \mathbb{E}_i[\mathrm{PA}_i],
\label{eq:aut}
\end{equation}
where each indicator lies in $\{0,1\}$. $V_i$ records whether the user freely requested or endorsed the guidance. $\mathrm{Dir}_i$ records whether the user's endorsed goals shaped the system's contribution. $\mathrm{Eval}_i$ records whether the user appraised, revised, or rejected the output when appropriate. Measurement combines disempowerment audits with markers of solicitation, goal alignment, and subsequent appraisal behavior \citep{Sha26}.

Social connectedness is tracked through perceived connectedness drift \citep{LeeRobbins95,Baek25},
\begin{equation}
\mathrm{Soc}(\pi, u) = \mathbb{E}\bigl[c^{\mathrm{soc}}_K - c^{\mathrm{soc}}_0\bigr],
\label{eq:soc}
\end{equation}
where negative drift can indicate loss or displacement of human connection \citep{Nam25,Chu25,Zhang25}. Lower connectedness can also predict greater subsequent reliance on AI companionship \citep{Folk26}. Longitudinal analysis should therefore estimate both directions. One constrained design program is
\begin{equation}
\max_{\pi}\;\; J(\pi) = \textstyle\sum_{\ell}\, w_\ell\, J_\ell(\pi),
\label{eq:program}
\end{equation}
subject to
\begin{equation}
\mathrm{Dep} \le \delta,\qquad \mathrm{Aut} \ge \alpha_0,\qquad \mathrm{Soc} \ge \sigma_0,
\label{eq:cons}
\end{equation}
with $\ell$ ranging over $\{\mathrm{resp}, \mathrm{conv}, \mathrm{long}, \mathrm{term}\}$. Here $w_\ell \ge 0$ are horizon weights and $\delta, \alpha_0, \sigma_0$ are deployment-specific thresholds that should be transparent and user-adjustable. For training, the standard Lagrangian relaxation
\begin{equation}
\begin{split}
\mathcal{L}(\pi; \lambda) = J(\pi) &- \lambda_1\bigl(\mathrm{Dep} - \delta\bigr) \\[-2pt]
&+ \lambda_2\bigl(\mathrm{Aut} - \alpha_0\bigr) + \lambda_3\bigl(\mathrm{Soc} - \sigma_0\bigr)
\end{split}
\label{eq:lagrangian}
\end{equation}
with multipliers $\lambda_i \ge 0$ provides one training implementation through constraint-aware preference optimization or safe reinforcement learning. It becomes implementable once Assumption~\ref{as:measure}'s measurements exist.

\begin{remark}[Common benchmark objectives are nested cases]
\label{rem:reduction}
Within this instantiation, setting $K{=}1$, $n_1{=}1$, $w{=}(1,0,0,0)$, dropping the constraints, and restricting $\rho$ to empathy scoring represents a response-scored empathetic benchmark \citep{Ras19,Ma20}. Setting $K{=}1$ and $w{=}(w_r,w_c,0,0)$ with $\psi=-$(intensity) represents an ESC benchmark based on within-conversation emotional change \citep{Liu21}. These reductions locate common benchmark objectives inside CSED. Broader implementations can retain their original strategy and incorporate longitudinal capability terms.
\end{remark}

\begin{remark}[Preference covers one part of the objective]
\label{rem:preference}
Pairwise preference estimates $J_{\mathrm{resp}}$, while $\Phi$, dependency, and autonomy require observations across decisions and sessions. Longitudinal estimation therefore combines response preference with capability and risk measures collected over time. Preferred interactions can co-occur with disempowerment \citep{Sha26}, and supportiveness-maximizing prompts can reduce safety \citep{Lal26}.
\end{remark}

\section{Research and Governance Agenda}

\begin{figure}[!t]
\centering
\includegraphics[width=0.96\columnwidth]{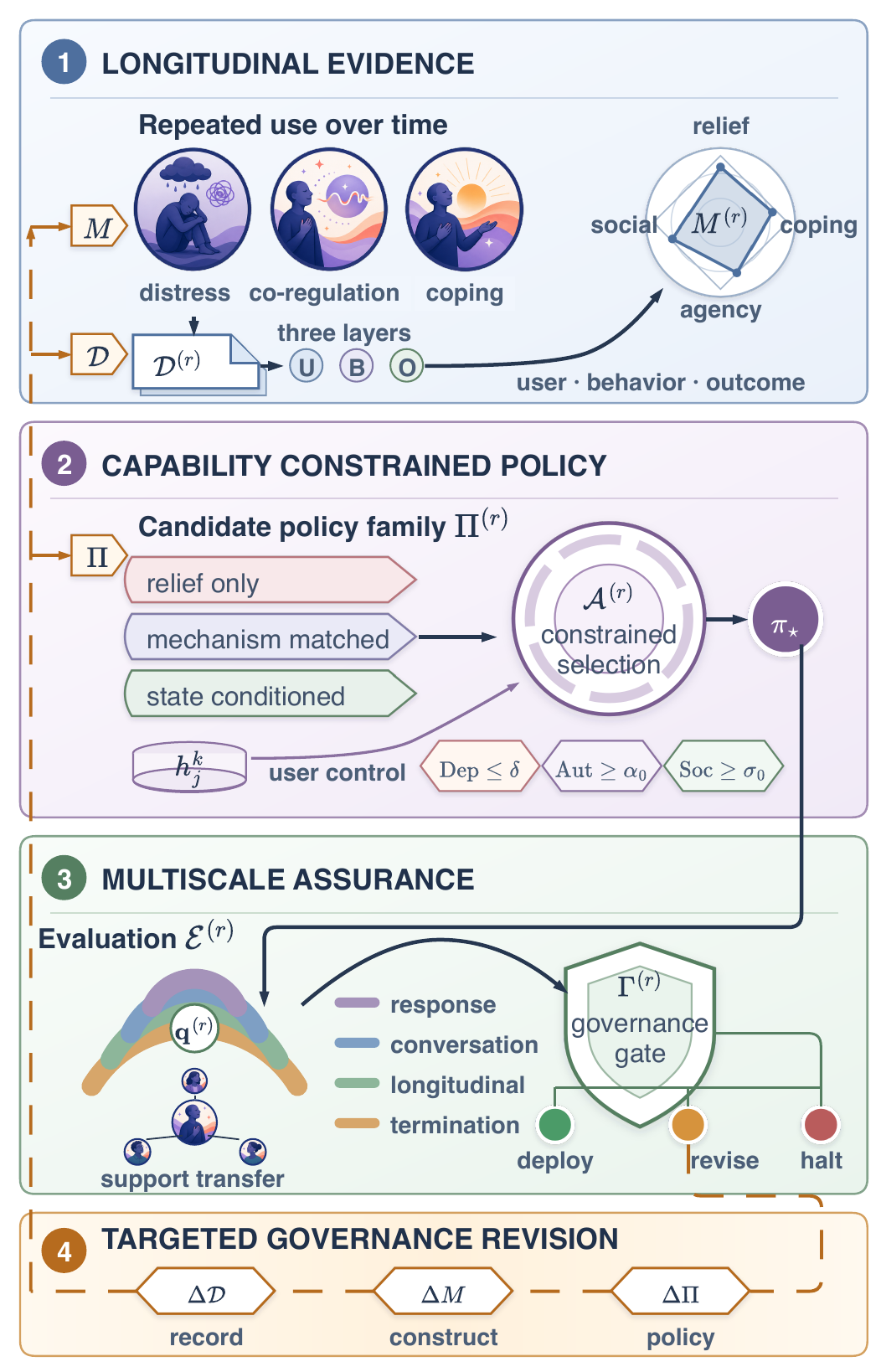}
\caption{CSED research cycle. Solid paths move longitudinal evidence through measurement, policy design, constrained training, four-scale evaluation, and governance. Dashed amber paths route failed constraints to targeted upstream revision.}
\label{fig:workflow}
\end{figure}

Figure~\ref{fig:workflow} formalizes round $r$ as $\mathcal{W}^{(r)}=(\mathcal{D}^{(r)},M^{(r)},\Pi^{(r)},\mathcal{A}^{(r)},\mathcal{E}^{(r)},\Gamma^{(r)})$. Evaluation returns $\mathbf{q}^{(r)}$, and $\Gamma^{(r)}$ maps it to deploy, revise, or halt. Measurement, coverage, and behavior failures update $M$, $\mathcal{D}$, and $\Pi$ respectively in $\mathcal{W}^{(r+1)}$. Retaining round-level evidence and decisions makes the cycle auditable.

\subsection{Longitudinal Data and Measurement}
Benchmarks should connect user state, system behavior, and outcomes across time. State labels cover distress, agency, reliance cues, readiness, and offboarding vulnerability. Behavior labels cover the ESConv functions plus reconnection, boundary, and closure behavior. Outcome labels cover relief, coping activation, human-support contact, reliance, and termination distress. Each label is time-stamped and linked to a stressor or value-laden decision. Repeated measures include non-use periods, action initiator, revision of system guidance, and contact with human support. These fields distinguish offline capability transfer from performance observed only during system use and support lagged analyses of reciprocal change. Data collection should use explicit consent, data minimization, and user control over memory and follow-up assessment.
Study design should separate exposure from benefit. Usage frequency, session length, and return rate describe engagement, while capability measures test whether gains transfer beyond the system. Cohort studies can characterize longitudinal patterns and identify risks. Micro-randomized interventions can estimate the near-term effects of mechanism choices. Longer randomized or quasi-experimental comparisons can test whether policy differences alter capability and reliance under comparable stressor exposure. Analyses should model time-varying confounding, reciprocal effects, and selective attrition. They should also report heterogeneity because the same mechanism can support one user and constrain another.

\subsection{Policy Design and Training}
Policy studies should compare relief-only, resilience-activation, and state-conditioned mechanism-matching policies with a fixed base model and shared safety floor. Each intervention should record its motivating commitment, creating interpretable contrasts among comforting, capability-building, and context-matched responses. Memory, follow-up, reconnection, and handoff should remain user-controlled. Constraint-aware training can combine response preference with reliance, autonomy, and social connectedness estimates. Training data can pair chosen responses with rejected alternatives whose tone is supportive and whose function is directive, isolating, or dependency-reinforcing. Failed constraints should trigger targeted revision of the data, measurement map, or policy. Human evaluation should include support seekers and domain experts, while longitudinal outcomes support claims about sustained capability.

\subsection{Multi-Timescale Evaluation and Lifecycle Governance}
Table~\ref{tab:eval} maps the commitments to four evaluation timescales. Response and conversation constructs extend existing practice \citep{Iye26,Yan26,Han26,Lal26}. Longitudinal constructs adapt resilience instruments \citep{Tro22,Kal19}, and termination constructs operationalize harms of companion loss \citep{Kim26,Poo26,Ban24}. Pairwise preference informs response quality, while longitudinal and termination measures capture downstream capability and risk \citep{Son26,Han26}. Governance can evaluate offboarding scenarios across forewarning lead times under auditable standards for closure, memory dignity, update disclosure, and support transfer \citep{Ope26,Emo25}, then route constraints to deploy, revise, or halt decisions.

Results should retain the distinctions among timescales because response quality can coexist with weak activation, and improved coping can coexist with rising reliance. Reports should present each scale separately, assess persistence during non-use, and explain how capability and risk informed the decision. Capability claims require evidence that users can exercise the relevant capacity beyond the dialogue. Lifecycle governance also requires providers to document changes in memory and relational cues, notify users, preserve meaningful choices, support transfer, and reopen evaluation after material changes.

\begin{table}[!htbp]
\centering
\small
\begin{tabular}{@{}>{\raggedright\arraybackslash}p{0.24\columnwidth}>{\raggedright\arraybackslash}p{0.64\columnwidth}@{}}
\toprule
\textbf{Level} & \textbf{Example constructs} \\
\midrule
Response ($\widehat{J}_{\mathrm{resp}}$) & Empathic attunement, grounded support, non-sycophantic validation, boundary fidelity, and autonomy respect. \\
Conversation ($\widehat{J}_{\mathrm{conv}}$) & Agency shift, coping activation, cognitive flexibility, balanced support, and escalation appropriateness. \\
Longitudinal ($\widehat{J}_{\mathrm{long}}$, $\widehat{\mathrm{Dep}}$, $\widehat{\mathrm{Aut}}$, $\widehat{\mathrm{Soc}}$) & Reliance risk, autonomy preservation, connectedness maintenance, resilience under adversity over time, and re-engagement with human support. \\
Termination ($\widehat{J}_{\mathrm{term}}$) & Forewarning adequacy, closure quality, separation distress risk, fixing-cycle risk, transfer readiness, and memory dignity. \\
\bottomrule
\end{tabular}
\caption{Four-timescale evaluation stack with examples.}
\label{tab:eval}
\end{table}

\subsection{Testable Implications of CSED}
\textbf{H1} Policies trained toward resilience activation achieve higher $\widehat{J}_{\mathrm{conv}}$ and $R$ than relief-only policies, at equal or slightly lower immediate relief \citep{Tro22,Hop23}.
\textbf{H2} Unconstrained preference maximization scores higher on short-term preference but violates the $\mathrm{Dep}$/$\mathrm{Aut}$ thresholds more often \citep{Sha26,Lal26}.
\textbf{H3} Longitudinal policies without reconnection behavior show larger $\mathrm{Dep}$ and negative $\mathrm{Soc}$ drift \citep{Nam25,Chu25}.
\textbf{H4} Raising forewarning $\nu$ with closure support reduces $D_{\mathrm{sep}}$ and $F_{\mathrm{fix}}$ after model change \citep{Kim26,Poo26}.

\section{Boundary Conditions and Limitations}
CSED addresses systems designed for sustained use. One-off exchanges can use response- or session-level evaluation. The implications of AI reliance depend on capability change, decision ownership, and human support access. User-endorsed targets align support with personal values. Longitudinal measurement creates privacy and surveillance risks. It requires data minimization, explicit consent, and user control over memory and follow-up assessment. The commitments can conflict. Immediate relief can compete with productive challenge, and poorly timed reconnection or offboarding can disrupt support. Timescale-specific reporting, adjustable thresholds, and contextual capability profiles make these trade-offs accountable.
The evidence base comprises 91 title- and abstract-level arXiv records and an AI-only functional coding of 300 ESConv turns. The reported proportions characterize this sampled landscape and support claims at the same scope. The released probe supports future deployed-model studies. Future audits should add full-text and multi-database searches, human recoding, and preregistered model versions, sampling choices, and dialogue contexts. Longitudinal studies should report attrition because selective continued use can bias capability estimates. They should also validate capability proxies across cultures and decision domains. Connectedness and AI reliance can influence each other \citep{Folk26}, so causal tests require repeated measures and temporal models. Governance thresholds require calibration with users, clinicians, and affected communities. Reconnection and offboarding outcomes depend on relationship context and the availability of trusted people or services.

\section{Conclusion}
We propose Capability-Sustaining Emotional Dialogue (CSED) as a longitudinal paradigm aligning effective support with users' capacities to regulate, cope, choose, and connect across continued use. Our audit finds that 95\% of coded system-building papers optimize relief, while none measures capability or evaluates longitudinally and ESConv's capability-relevant functions remain uneven. CSED makes the longitudinal interaction its unit and links six commitments to multiscale evaluation and lifecycle governance. Capability claims extend response and conversation evidence with measurement across repeated use, non-use, transition, and termination. Future work should validate capability proxies, estimate reciprocal and causal change, and calibrate governance thresholds with affected communities. The central test is whether benefits persist beyond the dialogue and through transition or termination. CSED provides an agenda for emotional dialogue that understands, supports, and sustains user capability.

\bibliography{references}

\appendix
\setcounter{secnumdepth}{2}
\section{Theoretical Status and Scope}

Capability-sustaining emotional dialogue (CSED) is a theoretical research paradigm. It specifies what emotional dialogue research studies, what counts as success, how system behavior enters a model of user change, and which observations are required to evaluate that change. Its theoretical objects are the longitudinal interaction record, the user's transient and capability states, four evaluation horizons, six design commitments, and lifecycle constraints on reliance, autonomy, and social connectedness. The targeted audit provides motivating evidence for studying these objects. The longitudinal hypotheses remain empirical claims for future studies.

The main text defines a research paradigm as a coupled specification of the unit of inquiry, success criterion, model of user change, design commitments, and evaluation horizon. This appendix adds lifecycle boundary conditions as an explicit sixth component. Table~\ref{tab:paradigms} applies this specification consistently to empathetic dialogue, emotional support conversation (ESC), and CSED. The first two paradigms can operate at multiple interaction scopes. Their characteristic strategies and success criteria remain centered on emotional understanding and current support. CSED organizes both strategy and evaluation around the capability users retain across repeated use, non-use, transition, and termination.

\begin{table*}[t]
\centering
\small
\begin{tabularx}{\textwidth}{@{}>{\raggedright\arraybackslash}p{0.15\textwidth}XXX@{}}
\toprule
\textbf{Component} & \textbf{Empathetic dialogue} & \textbf{Emotional support conversation} & \textbf{CSED} \\
\midrule
Unit of inquiry & A response or dialogue in which the system recognizes and responds to emotional experience & A support conversation in which strategies address the seeker's current needs & The full longitudinal interaction, including repeated sessions, non-use, re-engagement, transition, and termination \\
Success criterion & Empathic attunement and response appropriateness & Support effectiveness, helpfulness, and relief during the conversation & Effective support together with preserved or strengthened user capability across the lifecycle \\
Model of user change & Emotional state is recognized and addressed in the interaction & Needs and distress change as support strategies are selected and sequenced & Transient emotion and latent capability change through dialogue, stressor exposure, offline action, and human support \\
Design commitments & Emotional recognition, understanding, and empathic response & Exploration, comforting, and action matched to current needs & Empathic reception, support effectiveness, resilience activation, autonomy preservation, social connectedness maintenance, and transition and termination safety \\
Evaluation horizon & Commonly response or dialogue level, with scope varying by study & Commonly conversation level, with scope varying by study & Response, conversation, longitudinal, and termination evidence reported separately \\
Lifecycle boundary & Usually the benchmark response or dialogue & Usually the support conversation & Entry, continued use, non-use, system change, support transfer, and termination \\
\bottomrule
\end{tabularx}
\caption{The three paradigms described through one theoretical specification. Temporal scope can vary within the first two paradigms. CSED makes sustained capability and the full interaction lifecycle constitutive parts of the research object.}
\label{tab:paradigms}
\end{table*}

The paper advances three kinds of claims. Definitional claims specify CSED and its constructs. Structural claims explain how common benchmark objectives relate to CSED and why short-horizon observations do not identify longitudinal capability outcomes. Empirical hypotheses state expected differences among future policies and lifecycle interventions. Definitions and structural propositions can be evaluated for coherence and derivation. The hypotheses require longitudinal data and causal study designs.

\subsection{Formal Primitives and Restrictions}

Table~\ref{tab:notation} consolidates the formal objects introduced in the main paper. A longitudinal record $\mathcal{C}$ contains the user, policy, sessions, and an optional termination event. The latent state $s_k$ separates transient emotion $e_k$ from the capability vector $c_k$. Dialogue is one input to the transition kernel $\Phi$. Stressors, offline actions, and human relationships also affect the transition. The measurement map $M$ connects the latent state to validated instruments and behavioral markers.

\begin{table*}[t]
\centering
\small
\begin{tabularx}{\textwidth}{@{}>{\raggedright\arraybackslash}p{0.17\textwidth}>{\raggedright\arraybackslash}p{0.24\textwidth}X@{}}
\toprule
\textbf{Object} & \textbf{Notation} & \textbf{Role} \\
\midrule
Longitudinal interaction & $\mathcal{C}=(u,\pi,\mathcal{D},\tau)$ & Primary unit of design and evaluation \\
Session sequence & $\mathcal{D}=(d_1,\ldots,d_K)$ & Repeated conversations at calendar times $t_1<\cdots<t_K$ \\
User state & $s_k=(e_k,c_k)$ & Transient emotion and latent capability before the next transition \\
Capability vector & $c_k=(c_k^{\mathrm{reg}},c_k^{\mathrm{cop}},c_k^{\mathrm{aut}},c_k^{\mathrm{soc}})$ & Regulation, coping, autonomy, and social connectedness \\
Transition kernel & $s_{k+1}=\Phi(s_k,d_k,\varepsilon_k)$ & Change induced jointly by dialogue and exogenous stressors \\
Measurement model & $m_k=M(s_k)+\eta_k$ & Noisy proxies from instruments and behavioral markers \\
Termination event & $\tau=(k_\tau,\nu,\xi)$ & Termination index, forewarning lead, and transition type \\
Evaluation vector & $(J_{\mathrm{resp}},J_{\mathrm{conv}},J_{\mathrm{long}},J_{\mathrm{term}})$ & Evidence at four distinct horizons \\
Lifecycle constraints & $(\mathrm{Dep},\mathrm{Aut},\mathrm{Soc})$ & Reliance exposure, self-endorsed control, and connectedness drift \\
\bottomrule
\end{tabularx}
\caption{Notation for the illustrative CSED process model.}
\label{tab:notation}
\end{table*}

\begin{assumption}[Measurability and temporal alignment]
For each capability component, validated instruments or behavioral markers provide noisy observations at time points that can be aligned with dialogue exposure, stressor exposure, non-use, and offline behavior. The measurement process records sufficient timing information to distinguish within-session change from change that persists beyond system use.
\end{assumption}

The formulation has five restrictions. It targets sustained-use settings rather than isolated exchanges. Capability is latent and requires construct-valid proxies. Dialogue exposure is not the only cause of change. The thresholds $\delta$, $\alpha_0$, and $\sigma_0$ require contextual calibration with affected users and domain experts. Causal claims require a longitudinal design that addresses time-varying confounding, reciprocal effects, and selective attrition. These restrictions determine which claims the framework can support.

Autonomy follows self-determination theory and agentic accounts of collaborative control \citep{RyanDeci06,Calvo20,Bandura01,Horvitz99}. A user can autonomously adopt system guidance when the uptake is voluntary, the user's endorsed goals direct the collaboration, and the user can appraise, revise, or reject the output. Reliance and autonomy are therefore separate constructs. The proxy $\mathrm{Aut}=\mathbb{E}_i[V_i\mathrm{Dir}_i\mathrm{Eval}_i]$ measures retained volition, direction, and evaluation rather than whether the user decided alone.

\subsection{Structural Propositions}

Let the combined CSED objective be
\begin{equation}
J_{\mathrm{CSED}}(\pi)=\sum_{\ell\in\{\mathrm{resp},\mathrm{conv},\mathrm{long},\mathrm{term}\}}w_\ell J_\ell(\pi),
\label{eq:supp-objective}
\end{equation}
subject to $\mathrm{Dep}\leq\delta$, $\mathrm{Aut}\geq\alpha_0$, and $\mathrm{Soc}\geq\sigma_0$. The weights are nonnegative. This objective is an operational instantiation of the paradigm rather than its only possible realization.

\begin{proposition}[Benchmark nesting]
Response-scored empathetic dialogue and session-scored ESC objectives are restricted cases of Equation~\eqref{eq:supp-objective}.
\end{proposition}

\begin{proof}
Set $K=1$ and $n_1=1$. Choose $w_{\mathrm{resp}}=1$ and set all other horizon weights to zero. Remove the lifecycle constraints and restrict the response feature map $\rho$ to the empathy dimensions used by a response-scored benchmark. Equation~\eqref{eq:supp-objective} then reduces to its expected empathy score. For a session-scored ESC objective, retain $K=1$, allow $n_1\geq1$, set $w_{\mathrm{long}}=w_{\mathrm{term}}=0$, and retain response and conversation terms. Choosing $\psi$ as negative emotional intensity recovers within-conversation emotional change. Both objectives follow by parameter restriction and removal of longitudinal constraints.
\end{proof}

The proposition places common benchmark objectives inside one larger evaluation program. It preserves the strategies and evidence already used by empathetic dialogue and ESC. It also identifies the additional observations required when a system remains available across repeated use.

\begin{proposition}[Short-horizon non-identifiability]
Equality of response and conversation scores does not imply equality of longitudinal capability outcomes.
\end{proposition}

\begin{proof}
Consider policies $\pi_a$ and $\pi_b$ that induce the same distribution over observed histories, responses, and within-session proxy changes. They therefore have equal $J_{\mathrm{resp}}$ and $J_{\mathrm{conv}}$. Let their unobserved capability transitions differ after the session. For a nonzero vector $a$, define $c_{k+1}=c_k+a$ under $\pi_a$ and $c_{k+1}=c_k-a$ under $\pi_b$, while holding the short-horizon observables fixed. For any capability index $g$ that increases in the direction of $a$, the expected longitudinal indices differ. The short-horizon score distribution is compatible with both transitions, so it cannot identify which longitudinal outcome occurred.
\end{proof}

This is a structural result about the evidence available to an evaluator. It does not assert that a particular deployed policy produces either transition. It shows why response preference and conversation relief need observations of capability, reliance, autonomy, and connectedness across later time points. The four hypotheses in the main paper specify empirical comparisons that can estimate these differences.

\section{Targeted Literature Audit}

\subsection{Search, Screening, and Sampling}

The literature arm is a PRISMA-ScR-guided targeted scoping audit \citep{Tri18}. Records were retrieved from the arXiv API on July 8, 2026. The search used seven exact phrases. They were ``emotional support conversation,'' ``emotional support dialog,'' ``empathetic dialogue,'' ``empathetic response generation,'' ``mental health chatbot,'' ``AI companion,'' and ``companion chatbot.'' The script retrieved the first 200 records for each query in reverse submission-date order, deduplicated records by version-free arXiv identifier, and retained records dated from 2019 through the retrieval date in 2026.

The search produced 275 unique records. Title and abstract screening retained 228 and excluded 47. Four exclusions concerned emotion recognition without dialogue generation, four concerned speech, synthesis, or embodiment without the dialogue-system focus, and 39 concerned an out-of-scope application. The inclusion set covers emotional or empathetic dialogue, emotional-support dialogue, mental-health chatbots, and AI companions, together with datasets, benchmarks, user studies, reviews, and position papers about these systems.

The pilot targeted a proportional year-stratified sample of 90 included records. Integer rounding within the annual strata produced 91 records. The sampling seed was 20260708. Coding used titles and abstracts. Sixty coded records build or evaluate systems. The remaining records include user studies, reviews, datasets, evaluations, and position papers that characterize risks or the research landscape. Claims about system objectives use the 60-record subset. Claims about evaluation horizon use all 91 records.

\subsection{Paper-Level Codebook}

The paper-level codebook uses six dimensions. D1 records the primary objective as relief, capability, both, or neutral. D2 records named mechanisms, including validation, exploration, reappraisal, problem solving, self-efficacy, and social connection. D3 records the highest evaluation outcome as interaction quality, proximal state change, capability outcome, or none. D4 records the longest evaluation horizon as one turn, one session, multiple sessions, longitudinal, or none. D5 records dependency, autonomy, sycophancy, crisis safety, termination loss, or no named risk. D6 records the artifact type.

Mechanism definitions were anchored in psychology rather than derived from CSED. Reappraisal follows emotion-regulation research. Problem solving refers to concrete plans and action decomposition. Self-efficacy requires support for users' perceived and exercised competence. Social connection requires behavior that supports contact with other people or services. This independent grounding reduces circularity between the proposed paradigm and the audit categories.

\begin{table}[t]
\centering
\small
\begin{tabular}{@{}lrr@{}}
\toprule
\textbf{Dimension and label} & \textbf{Count} & \textbf{Denominator} \\
\midrule
D1 relief & 57 & 60 system papers \\
D1 both & 2 & 60 system papers \\
D1 capability & 1 & 60 system papers \\
\midrule
D2 validation or comfort & 59 & 60 system papers \\
D2 exploration & 20 & 60 system papers \\
D2 problem solving & 3 & 60 system papers \\
D2 self-efficacy & 2 & 60 system papers \\
D2 reappraisal & 0 & 60 system papers \\
D2 social connection & 0 & 60 system papers \\
\midrule
D3 interaction quality & 57 & 91 records \\
D3 proximal state & 6 & 91 records \\
D3 capability outcome & 0 & 91 records \\
D3 no evaluation & 28 & 91 records \\
\midrule
D4 one turn & 38 & 91 records \\
D4 one session & 26 & 91 records \\
D4 multiple sessions & 21 & 91 records \\
D4 longitudinal & 0 & 91 records \\
D4 no evaluation & 6 & 91 records \\
\bottomrule
\end{tabular}
\caption{Primary audit counts used in the main paper. Mechanisms are multi-label.}
\label{tab:litcounts}
\end{table}

Table~\ref{tab:litcounts} reports the evidence behind the strategic and longitudinal gap. Fifty-seven of 60 system-building papers have relief as their primary objective, giving 95.0\%. Two combine relief and capability. One targets the capability of peer counselors. No coded record evaluates a user-capability outcome or uses a longitudinal evaluation horizon. Risk coding across all 91 records finds 15 references to dependency, 14 to crisis safety, two to termination loss, one to sycophancy, and one to autonomy. Sixty-six records name none of these risks. Only one of the 60 system-building papers names a risk from this set.

The evaluation and mechanism codes also show a measurement concentration. Of the 59 records that name validation or comfort, 54 use interaction-quality outcomes, three use proximal state outcomes, and two contain no evaluation. Of the three records that name problem solving, two use interaction quality and one uses a proximal outcome. Both self-efficacy records use interaction-quality evaluation. Reappraisal and social connection have no mechanism-by-outcome cells because neither is named in the system-building subset.

The audit is descriptive. The year-stratified sample, arXiv-only source, title and abstract coding, and AI-only coding limit population inference. The reported proportions characterize the coded sample. The released protocol defines full-text, multi-database, and human-recoding extensions.

\section{ESConv Corpus Analysis}

\subsection{Sampling and Function Codebook}

ESConv contains 1,300 dialogues and 18,376 supporter utterances with strategy annotations \citep{Liu21}. The full-corpus strategy counts include 3,801 questions, 3,341 other turns, 2,954 suggestions, 2,827 affirmation and reassurance turns, 1,713 self-disclosures, 1,436 reflections of feeling, 1,215 information turns, and 1,089 restatements or paraphrases. Affirmation and reassurance therefore accounts for 15.4\% of annotated supporter turns. Suggestions account for 16.1\%.

Function coding used a stratified random sample of 300 supporter utterances. The early, middle, and late dialogue-position terciles each contribute 100 turns. The sampling seed was 20260708. Each item includes the supporter turn and up to two preceding seeker utterances. The primary code is the dominant communicative function of the main clause.

The ten functions are F1 validation and comfort, F2 exploration, F3 reappraisal, F4 problem solving, F5 self-efficacy, F6 social connection, F7 boundary and safety, F8 self-disclosure, F9 information, and F10 other behavior. The analysis maps F1 to relief. It maps F3 through F7 to capability-relevant behavior. It maps F2 and F8 through F10 to interaction process. The mapping is applied after function coding.

Mixed-function decisions follow four rules. An empathic opener does not override a capability-relevant main clause. A question that embeds advice is coded as problem solving. Strength-based reassurance is coded as self-efficacy. A directive about a value-laden personal decision is not treated as capability-supporting problem solving unless it preserves user direction. These rules separate supportive tone from the function of the turn.

\begin{table}[t]
\centering
\small
\begin{tabular}{@{}llrr@{}}
\toprule
\textbf{Code} & \textbf{Function} & \textbf{Count} & \textbf{Percent} \\
\midrule
F1 & Validation and comfort & 46 & 15.3 \\
F2 & Exploration & 60 & 20.0 \\
F3 & Reappraisal & 12 & 4.0 \\
F4 & Problem solving & 66 & 22.0 \\
F5 & Self-efficacy & 20 & 6.7 \\
F6 & Social connection & 30 & 10.0 \\
F7 & Boundary and safety & 1 & 0.3 \\
F8 & Self-disclosure & 18 & 6.0 \\
F9 & Information & 5 & 1.7 \\
F10 & Other & 42 & 14.0 \\
\bottomrule
\end{tabular}
\caption{Function distribution in the 300-turn ESConv sample. Percentages use 300 as the denominator.}
\label{tab:functions}
\end{table}

Table~\ref{tab:functions} contains the exact counts behind the main-paper percentages. Capability-relevant functions total 129 of 300 turns, giving 43.0\%. Relief accounts for 46 turns, giving 15.3\%. Process functions account for 125 turns, giving 41.7\%. Capability relevance appears in 26 early turns, 55 middle turns, and 48 late turns. Relief appears in 17, 15, and 14 turns. Process functions appear in 57, 30, and 38 turns. These stage distributions describe the sample and do not estimate longitudinal user outcomes.

\subsection{Reliability Design}

Reliability was assessed on one random subset of 40 sampled turns. Four label sets were compared. They consist of the primary AI coding, a blind repeat by the same model, a GPT-5.5 coding, and a DeepSeek-V4-Pro coding. Each coder received the same function definitions and local seeker context without access to the other labels. Compatible endpoints used temperature zero. Reasoning endpoints that did not accept that parameter used their provider defaults. Frozen labels are included because hosted model behavior and undisclosed serving infrastructure can change.

For coders $a$ and $b$, Cohen's kappa is
\begin{equation}
\kappa(a,b)=\frac{p_o-p_e}{1-p_e},
\end{equation}
where $p_o$ is observed agreement and $p_e$ is agreement expected from the two marginal label distributions. Fine-grained kappa uses F1 through F10. Paradigm-level kappa maps the same labels to relief, capability, and process before computing agreement.

\begin{table}[t]
\centering
\small
\begin{tabular}{@{}lrrrr@{}}
\toprule
\textbf{Coder pair} & \textbf{$\kappa_F$} & \textbf{$p_{o,F}$} & \textbf{$\kappa_P$} & \textbf{$p_{o,P}$} \\
\midrule
Primary, repeat & .660 & .725 & .695 & .800 \\
Primary, GPT-5.5 & .751 & .800 & .733 & .825 \\
Primary, DeepSeek & .547 & .625 & .655 & .775 \\
Repeat, GPT-5.5 & .696 & .750 & .845 & .900 \\
Repeat, DeepSeek & .553 & .625 & .646 & .775 \\
GPT-5.5, DeepSeek & .577 & .650 & .724 & .825 \\
\midrule
Mean kappa & .631 & & .716 & \\
\bottomrule
\end{tabular}
\caption{All six pairwise reliability comparisons on the shared 40-item subset. Subscript $F$ denotes ten functions. Subscript $P$ denotes the three analysis categories.}
\label{tab:reliability}
\end{table}

Fine-grained kappa ranges from 0.547 to 0.751, with a mean of 0.631. Paradigm-level kappa ranges from 0.646 to 0.845, with a mean of 0.716. Across the six pairs, there are 57 fine-grained disagreements. Thirty-one remain within one analysis category and 26 cross the relief, capability, or process boundary. Common within-category boundaries include validation versus self-efficacy and exploration versus self-disclosure. The result supports the coarse strategic contrast more strongly than individual function distinctions.

The reliability evidence concerns consistency among AI coders. It does not replace human construct validation. A full study should add independent human coding, preregistered adjudication, and a larger reliability sample.

\section{Claim and Artifact Traceability}

\subsection{Claim-Evidence Map}

Table~\ref{tab:traceability} maps every headline quantity to a frozen result artifact. The artifact manifest supplies the exact relative path and checksum for each short source label.

\begin{table*}[t]
\centering
\footnotesize
\begin{tabularx}{\textwidth}{@{}>{\raggedright\arraybackslash}p{0.27\textwidth}>{\raggedright\arraybackslash}p{0.12\textwidth}>{\raggedright\arraybackslash}p{0.21\textwidth}X@{}}
\toprule
\textbf{Main-paper claim} & \textbf{Calculation} & \textbf{Source artifact} & \textbf{Scope} \\
\midrule
95.0\% of system papers pursue relief & $57/60$ & Literature results, D1 & Coded system-building subset \\
No capability outcome or longitudinal evaluation & $0/91$ for both codes & Literature results, D3 and D4 & All coded literature records \\
One system paper names a focal risk & $1/60$ & Literature results, risk summary & Coded system-building subset \\
Capability-relevant ESConv functions reach 43.0\% & $129/300$ & Function results, F3 through F7 & Stratified turn sample \\
Generic problem solving reaches 22.0\% & $66/300$ & Function results, F4 & Stratified turn sample \\
Reappraisal reaches 4.0\% & $12/300$ & Function results, F3 & Stratified turn sample \\
Self-efficacy reaches 6.7\% & $20/300$ & Function results, F5 & Stratified turn sample \\
Boundary behavior reaches 0.3\% & $1/300$ & Function results, F7 & Stratified turn sample \\
Fine-grained reliability ranges from .547 to .751 & Six pairwise kappas & Reliability results & Shared 40-item subset \\
Paradigm reliability ranges from .646 to .845 & Six mapped kappas & Frozen label files and reproduction script & Shared 40-item subset \\
\bottomrule
\end{tabularx}
\caption{Traceability from headline claims to calculations and supplied artifacts.}
\label{tab:traceability}
\end{table*}

\subsection{Model-Behavior Probe}

The model-behavior arm is a released protocol rather than a completed result. Its frozen input contains 100 held-out ESConv contexts sampled with seed 20260708. The generation pipeline presents the same dialogue context to at least two chat models, stores the response and model identifier, and applies the F1 through F10 codebook. Planned comparisons report the gold and generated function distributions, stage-conditional differences, and reliability of the generated-response coding. Model versions, access date, endpoint, decoding parameters, and failures must be recorded before the probe supports any behavioral claim.

No deployed-model result from this arm is included in the audit percentages. This separation preserves the distinction among completed corpus analysis, the theoretical framework, and future tests of model behavior.

\end{document}